\documentclass{article}

    \PassOptionsToPackage{numbers, compress}{natbib}

\usepackage[accepted]{icml2023}

\usepackage[utf8]{inputenc} 
\usepackage[T1]{fontenc}    
\usepackage{hyperref}       
\usepackage{url}            
\usepackage{booktabs}       
\usepackage{amsfonts}       
\usepackage{nicefrac}       
\usepackage{microtype}      
\usepackage{xcolor}         
\usepackage{multirow}

\usepackage{graphicx} 
\usepackage{amsmath} 
\usepackage{fancyhdr} 
\usepackage{caption,subcaption}
\usepackage{tikz,graphics,float,epsf}
\usepackage{xcolor}
\usepackage{cancel}
\usepackage{wrapfig}

\usepackage{hyperref}
\hypersetup{colorlinks, citecolor=blue}

\usepackage[english]{babel}
\usepackage{csquotes}
\usepackage{enumitem}

\usepackage{amsfonts} 
\usepackage{natbib}
\usepackage{amssymb}
\usepackage{amsthm}
\usepackage{amsmath}
\usepackage{mathtools}
\usepackage{xspace}
\usepackage{array}
\newcolumntype{P}[1]{>{\centering\arraybackslash}p{#1}}

\theoremstyle{plain}
\newtheorem{theorem}{Theorem}[section]
\newtheorem{proposition}[theorem]{Proposition}

\theoremstyle{definition}

\newtheorem{assumption}[theorem]{Assumption}
\theoremstyle{remark}

\DeclareMathOperator*{\argmax}{arg\,max}

\usepackage{soul}

\newcommand{\piaux}{\pi_{R}\xspace}

\definecolor{es-blue}{rgb}{0,0.4,0.8}

\icmltitlerunning{TGRL: An Algorithm for Teacher Guided Reinforcement Learning}

\begin{document}
\twocolumn[
\icmltitle{TGRL: An Algorithm for Teacher Guided Reinforcement Learning}




\begin{icmlauthorlist}
\icmlauthor{Idan Shenfeld}{MIT}
\icmlauthor{Zhang-Wei Hong}{MIT}
\icmlauthor{Aviv Tamar}{Technion}
\icmlauthor{Pulkit Agrawal}{MIT}
\end{icmlauthorlist}

\icmlaffiliation{MIT}{Improbable AI Lab, Massachusetts Institute of Technology, Cambridge, USA}
\icmlaffiliation{Technion}{Technion - Israel Institute of Technology, Haifa, Israel}
\icmlcorrespondingauthor{Idan Shenfled}{idanshen@mit.edu}

\icmlkeywords{Machine Learning, ICML}

\vskip 0.3in
]
\printAffiliationsAndNotice{} 

\begin{abstract}
Learning from rewards (i.e., reinforcement learning or RL) and learning to imitate a teacher (i.e., teacher-student learning) are two established approaches for solving sequential decision-making problems. To combine the benefits of these different forms of learning, it is common to train a policy to maximize a combination of reinforcement and teacher-student learning objectives. However, without a principled method to balance these objectives, prior work used heuristics and problem-specific hyperparameter searches to balance the two objectives. We present a \textit{principled} approach, along with an approximate implementation for \textit{dynamically} and \textit{automatically} balancing when to follow the teacher and when to use rewards. The main idea is to adjust the importance of teacher supervision by comparing the agent's performance to the counterfactual scenario of the agent learning without teacher supervision and only from rewards. If using teacher supervision improves performance, the importance of teacher supervision is increased and otherwise it is decreased. Our method, \textit{Teacher Guided Reinforcement Learning} (TGRL), outperforms strong baselines across diverse domains without hyper-parameter tuning. The code is available at \url{https://sites.google.com/view/tgrl-paper/}

\end{abstract}

\section{Introduction}
\label{sec:intro}

In Reinforcement Learning (RL), an agent learns decision-making strategies by executing actions, receiving feedback in the form of rewards, and optimizing its behavior to maximize cumulative rewards. Such learning by trial-and-error can be challenging, particularly when rewards are sparse, or under partial observability~\cite{ madani1999undecidability, papadimitriou1987complexity}. A more data-efficient learning method is to directly supervise the agent with correct actions obtained by querying a \textit{teacher}, as exemplified by the imitation learning algorithm called DAgger~\cite{ross2011reduction}. Learning to mimic a teacher is significantly more data-efficient than reinforcement learning because it avoids the need to explore the consequences of different actions. 

However, learning from a teacher can be problematic when the teacher is sub-optimal or when it's impossible to perfectly mimic the teacher. In the first problematic case of a sub-optimal teacher,  because the agent attempts to mimic the teacher's actions perfectly, its performance is inherently limited by the teacher's performance. Developing methods for training agents that surpass their sub-optimal teachers is an active research area~\cite{agarwal2022beyond, kurenkov2019ac, rajeswaran2017learning}.
The second problem occurs when the agent is unable to mimic the teacher. It can happen in the common scenario when the teacher chooses actions based on \textit{privileged information} unavailable to the agent.
For example, the teacher may have access to additional sensors when training in simulation~\cite{lee2020learning, chen2021system, margolis2021jumping}, external knowledge bases~\cite{zhang2020deep}, or accurate state estimates during training~\cite{levine2015end}. 

In some scenarios, the agent can make up for the information gap with respect to the teacher by accumulating information from a history of observations~\cite{kumor2021sequential, swamy2022sequence}. However, in the most general scenario, just using the history is insufficient, and the agent must take information-gathering actions (i.e. explore) to acquire the information being used by the teacher before it can mimic it.
However, since the teacher never performs information-gathering actions, the agent cannot learn such actions by mimicking the teacher. As an example, consider the "Tiger Door" environment illustrated in Figure~\ref{fig:tiger_door}~\cite{littman1995learning, warrington2021robust}. The agent is placed in a maze with a goal cell (green), a trap cell (blue), and a button (pink). Reaching the goal and trap cells provide positive and negative rewards, respectively. The location of the goal and trap cells randomly switch locations every episode. The teacher is aware of the location of all cells, whereas the agent (or the student) cannot observe the goal/trap cell locations. Instead, the student can go to the pink button, an action that reveals the goal location. In this setup, the goal-aware teacher takes action to directly reach the goal. However, the student must deviate from the teacher's actions to reach the pink button -- a behavior that cannot be learned by imitation.    

\begin{figure}[t!]
    \centering
    \includegraphics[width=\columnwidth]{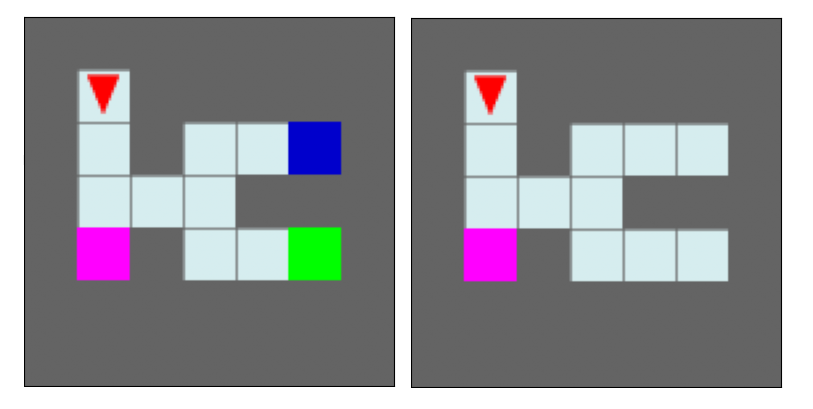}
    \caption{The Tiger Door environment. On the left is the teacher's observation, where the goal cell (in green) and the trap cell (in blue) are perceptible. On the right is the student's observation, where these cells are not visible, but there is a pink button; touching which reveals the other cells.
    }
    \label{fig:tiger_door}
\end{figure}

Consider the general scenario where the agent's optimality is measured by the rewards it accumulates. Both when the teacher is sub-optimal and when it cannot be mimicked, trying to imitate the teacher will result in sub-optimal policies. In these scenarios, a student with access to a reward function can benefit by jointly learning from both the reward and the teacher's supervision. Learning from rewards provides an incentive for the agent to deviate from a sub-optimal teacher to outperform it or carry out information gathering when learning from a privileged teacher. Thus, by combining both forms of learning, the agent can leverage the teacher's expertise to learn quickly but also try different actions to check if a better policy can be found. The balance between when to follow the teacher and use rewards is delicate and can substantially affect the performance (i.e., total accumulated rewards) of the learned policy. In the absence of a principled method to balance the two objectives, prior work resorted to task-specific hyperparameter tuning~\citep{weihs2021bridging, nguyen2022leveraging, agarwal2022reincarnating}.

In this work, we present a principled solution to automatically balance learning from rewards and a teacher. Our main insight is that supervision from the teacher should only be used when it improves performance compared to learning solely from reward. To realize this, in addition to training the \textit{main} policy that learns from both rewards and the teacher, we also train a second \textit{auxiliary} policy that learns the task by only optimizing rewards using reinforcement learning. At every training step, our algorithm compares the two policies. If the \textit{main} policy performs better, it indicates that utilizing the teacher is beneficial and the importance of learning from the teacher is increased. However, if the auxiliary policy performs better, the importance of the teacher's supervision in the main policy's objective is decreased. We call this algorithm for automatically adjusting the balance of imitation and RL objectives as \textit{Teacher Guided Reinforcement Learning (TGRL)}. 

We empirically evaluate TGRL on a range of tasks where learning solely from a teacher is inadequate and focus primarily on scenarios with a privileged teacher. The results show that TGRL is either comparable or outperforms existing approaches without the need for manual hyper-parameter tuning. The most challenging task we consider is robotic in-hand re-orientation of objects using only touch sensing. The superior performance of TGRL demonstrates its applicability to practical problems. Finally, we also present experiments showing the effectiveness of TGRL in learning from sub-optimal teachers.  

\section{Preliminaries}
\label{sec:prelim}

\textbf{Reinforcement learning (RL).} We consider the interaction between the agent and the environment as a discrete-time Partially Observable Markov Decision Process (POMDP)~\citep{kaelbling1998planning} consisting of state space $\mathcal{S}$, observation space $\Omega$, action space $\mathcal{A}$, state transition function $\mathcal{T}: \mathcal{S}\times\mathcal{A} \to \Delta(\mathcal{S})$, reward function $R:\mathcal{S}\times\mathcal{A}\to \mathbb{R}$, observation function $\mathcal{O}:\mathcal{S}\to\Delta(\Omega)$, and initial state distribution $\rho_0:\Delta(\mathcal{S})$. The environment is initialized at an initial state $s_0\sim \rho_0$. At each timestep $t$, the agent observes the observation $o_t \sim O(\cdot|s_t), o_t \in \Omega$, takes action $a_t$ determined by the policy $\pi$, receives reward $r_t = R(s_t, a_t)$, transitions to the next state $s_{t+1} \sim \mathcal{T}(\cdot|s_t, a_t)$, and observes the next observation $o_{t+1} \sim O(\cdot|s_{t+1})$. The goal of RL~\citep{sutton2018reinforcement} is to find the optimal policy $\pi^*$ maximizing the expected cumulative rewards (i.e., expected return). Since the agent has access only to the observations and not to the underlying states, seminal work showed that the optimal policy may depend on the history of observations $\tau_t: \{o_0, a_0, o_1,a_1...o_t\}$, and not only on the current observation $o_t$~\citep{kaelbling1998planning}. Our aim is finding the optimal policy $\pi^*:\tau \to \Delta (\mathcal{A})$ that maximizes the following objective:
\begin{equation}
\label{eq:env_obj}
    \pi^* = \argmax_{\pi} J_R(\pi) := \mathbb{E}\left[ \sum^\infty_{t=0}\gamma^{t} r_t \right].
\end{equation}

\textbf{Teacher-Student Learning (TSL).} Suppose the agent (also referred to as the \textit{student} in this paper) has access to a \textit{teacher} that computes actions, $a_t^T \sim \bar{\pi}(\cdot | o^T_t)$, using an observation space that may be different from the student, $o^T_t \sim \tilde{O}(\cdot|s_t); o^T_t \in \Omega^T$. We are agnostic to how the teacher is constructed. In general, it's a black box that can be queried by the student for actions at any state the student encounters during training. Given such a $\bar{\pi}$, we aim to train the student policy, $\pi_{\theta}(\cdot | o_t)$, operating from observations, $o_t \in \Omega$. 

A straightforward way to train the student is to use supervised learning to match the teacher's actions~\cite{argall2009survey}, $\max_{\theta} \mathbb{E}_{\bar{\pi}} \log \pi_{\theta}(a_t^T | o_t)$, where $o_t$ is the student's observation and $a_t^T$ is the teacher's action computed from its observation, $o_t^{T}$, corresponding to the state $s_t$ obtained by rolling out the teacher. However, recent work found that better student policies can be learned by using reinforcement learning to maximize the sum of rewards, where the reward is computed as the cross-entropy between the teacher and student's action distributions~\citep{czarnecki2019distilling}. This leads to the following optimization problem:  

\begin{equation}
\label{eq:teacher_obj}
\max_\pi {J_I(\pi)} := \max_\pi {\mathbb E \left[ - \sum_{t=0}^H {\gamma^t H^X_t(\pi|\bar{\pi})} \right]}
\end{equation}
where $H^X_t(\pi|\bar{\pi})=-\mathbb E_{a\sim \pi(\cdot|\tau_t)}[\log \bar{\pi}(a|o^T_t)]$ is the Shannon cross-entropy, and for convenience in notation, $\pi_{\theta}$ is denoted as $\pi$. This objective is similar to DAgger~\cite{ross2011reduction} in optimizing the learning objective using the data collected by the student.

\textbf{Problems in Teacher-Student Learning}. To understand the problems in TSL, consider the recent result that implies that a student trained with TSL learns the statistical average of the teacher's actions for each observable state $o\in\Omega$: 
\begin{proposition} \label{pro:il_policy}
In the setting described above, denote $\pi^{TSL} = \argmax_\pi {J_I(\pi)}$ and $f(o^T):\Omega_T\rightarrow\Omega$ as the function that maps the teacher's observations to the student's observations. Then, for any $o^T\in \Omega_T$
with $o=f(o^T)$, we have that $\pi^{TSL}(o)=\mathbb E [\bar{\pi}(s)|o=f(o^T)]$. 
\end{proposition}
\begin{proof}
See \citep{weihs2021bridging} proposition 1 or \citep{warrington2021robust} theorem 1.
\end{proof}
The two problems due to Proposition~\ref{pro:il_policy}: (i) Since the student's actions are the statistical average of the teacher's actions, it cannot outperform a sub-optimal teacher as there is no incentive to explore actions other than the teacher's. 
(ii) If the difference in observation spaces between the teacher and student is large, learning the statistical average can lead to sub-optimal performance. This is because the student cannot distinguish two different teachers' observations 
that appear identical in the student's observation space. As a result, the student policy does not mimic the teacher, but instead learns the \textit{average action}, which can lead to sub-optimal performance (Eq.~\ref{eq:env_obj})~\cite{kumor2021sequential, swamy2022sequence}. For example, in the Tiger Door environment, the student will follow the teacher until the second intersection (where the corridor splits into two paths for the two possible goal locations). The teacher policy takes a left or right action depending on where the goal is. Because the student does not observe the goal, it will learn to mimic the teacher's policy by assigning equal probability to actions leading to either of the sides. This policy is sub-optimal since the student will reach the goal only in $50\%$ of trials.  

\section{Method}
\label{sec:method}

As Teacher-Student Learning can lead to a sub-optimal student, to outperform the teacher, the student needs to explore actions different from the teacher to find a better policy. We assume that the student has access to task rewards in addition to a teacher. This reward function can guide the exploratory process by determining when deviating from the teacher is fruitful. Following prior work~\citep{czarnecki2019distilling, nguyen2022leveraging, agarwal2022reincarnating}, we consider the scenario of the student learning from a combination of reinforcement (Equation~\ref{eq:env_obj}) and teacher-student (Equation~\ref{eq:teacher_obj}) learning objectives:
\begin{equation}
\label{eq:ea_obj}
\max_\pi {J_{R+I}(\pi, \alpha)}=\max_\pi {\mathbb E \left[ \sum_{t=0}^H {\gamma^t(r_t-\alpha H^X_t(\pi|\bar{\pi}))}\right]}
\end{equation}

\normalsize
where $\alpha$ is the  balancing coefficient between the RL and imitation learning objectives. The joint objective can also be expressed as: 
$J_{R+I}(\pi, \alpha)=J_{R}(\pi)+\alpha J_{I}(\pi)$. Here, $J_{I}(\pi)$, can also be interpreted as a form of reward shaping~\cite{ng1999policy}, where the agent is negatively rewarded for taking actions that differ from the teacher's action.

As the balancing coefficient between the task reward and the teacher guidance, the value of $\alpha$ greatly impacts the algorithm's performance. A low $\alpha$ limits the guidance the student gets from the teacher, resulting in the usual challenges of learning solely from rewards. A high value of $\alpha$ can lead to excessive reliance on a sub-optimal teacher resulting in sub-optimal performance. 
Without a principled way to choose $\alpha$, a common practice is to find the best value of $\alpha$ by conducting a separate and extensive hyperparameter search for every task~\citep{schmitt2018kickstarting, nguyen2022leveraging}. 
Besides the inefficiency of such search, as the agent progresses on a task, the amount of guidance it needs from the teacher can vary. Therefore, a constant $\alpha$ may not be optimal throughout training. Usually, the amount of guidance the student needs diminishes along the training process, but the exact dynamics of this trade-off are task-dependent, and per-task tuning is tedious, undesirable, and often computationally infeasible.

\begin{figure*}[!h]
    \centering
    \includegraphics[width=\linewidth]{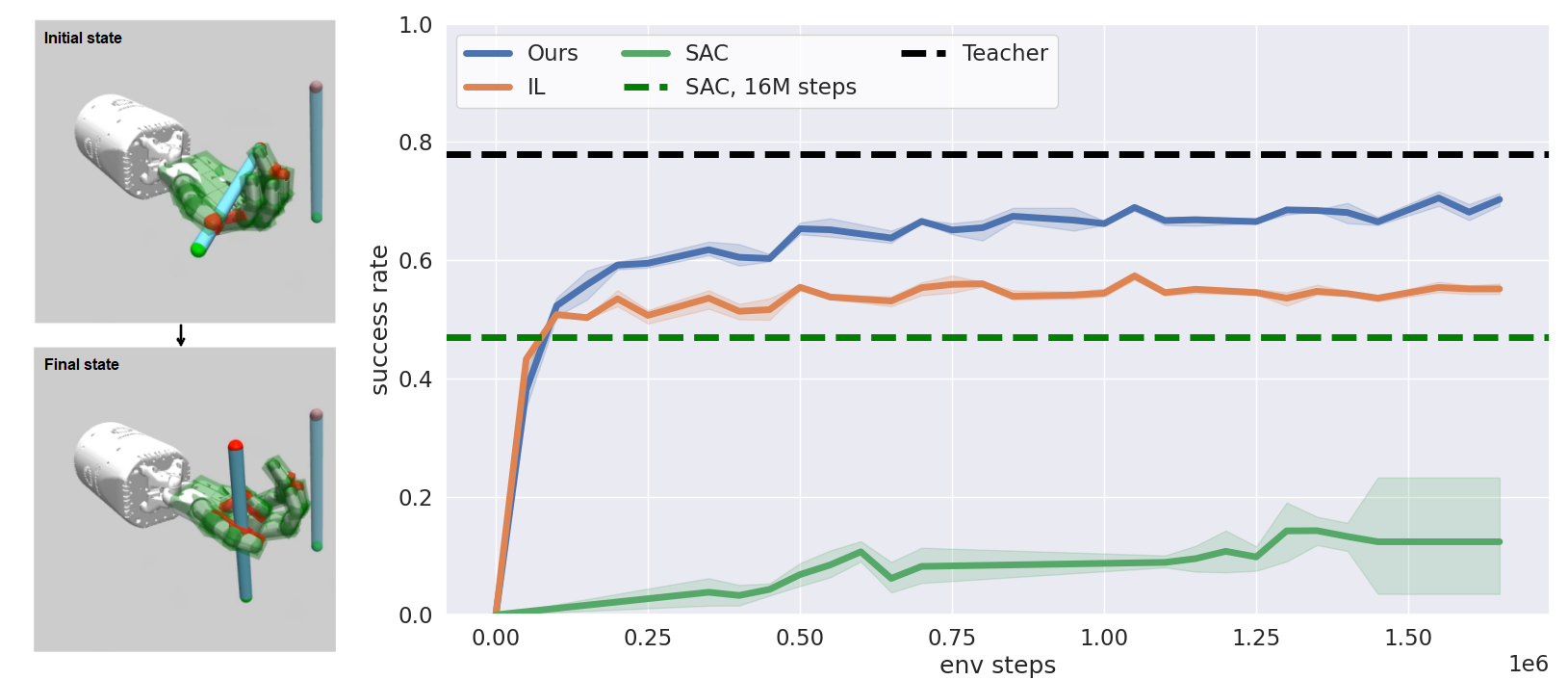}
    \vspace{-20pt}
    \caption{Success rate of a pen reorientation task by Shadow Hand robot, using tactile sensing only. While vanilla reinforcement learning takes a long time to converge, and Teacher-Student methods lead to a major drop in performance compared to the teacher, our algorithm is able to solve the task with reasonable sample efficiency.}
    \vspace{-10pt}
    \label{fig:hand_manipulation}
\end{figure*}

\subsection{Teacher Guided Reinforcement Learning (TGRL)}
Our notion of an optimal policy is one that achieves maximum cumulative task reward, and reinforcement learning optimizes this objective directly. Therefore, the teacher's supervision should only be used when it helps achieve better performance than just using task rewards. This idea is implemented by adding the following constraint: the performance of the policy learning from both rewards and teacher (i.e., the \textit{main} policy) must be at par or outperform a policy trained using only task rewards (i.e., the \textit{auxiliary} policy). Hence, our optimization problem becomes: 
\begin{equation}
\label{eq:const_prob}
\max _{\pi}{J_{R+I}(\pi, \alpha)} \quad \textrm{s.t.} \quad J_R(\pi) \geq J_R(\piaux)
\end{equation}
where $\piaux$ is the \textit{auxiliary} policy trained only using task reward (Eq.~\ref{eq:env_obj}). Overall, our algorithm iterates between improving the auxiliary policy  by solving $\max _{\piaux} {J_R(\piaux)}$ and solving the constrained problem in Equation~\ref{eq:const_prob}. A recent paper \citep{chen2022redeeming} used a similar constraint in another context, to balance between exploration and exploitation in conventional RL. More formally, for $i=1,2,…$ we iterate between two stages:

\begin{enumerate}
\item Partially solving $\piaux^i=\argmax _{\piaux} {J_R(\piaux)}$ to get an updated estimate $J_R(\piaux^i)$.
\item Solving the $i^{th}$ optimization problem:
\begin{equation}
\label{eq:primal_prob}
\max _{\pi}{J_{R+I}(\pi, \alpha)} \,\, \textrm{subject to} \,\,\, J_R(\pi) \geq J_R(\piaux^i)
\end{equation}
\end{enumerate}

The constrained optimization problem in Equation~\ref{eq:primal_prob} is solved using the dual lagrangian method, which has worked well in the reinforcement learning~\citep{tessler2018reward, bhatnagar2012online}. Using the Lagrange duality, we transform the constrained problem into an unconstrained min-max optimization problem. The dual problem corresponding to the primal problem in Equation~\ref{eq:primal_prob} is:
$$
\min_{\lambda \geq 0}{\max_{\pi}}\left[ J_{R+I}(\pi, \alpha) +\lambda \left( J_R(\pi) - J_R(\piaux) \right) \right] =
$$
\begin{equation}
\label{eq:dual_prob}
\min_{\lambda \geq 0}{\max_{\pi}} \left[ (1+\lambda) J_{R+I}(\pi, \frac{\alpha}{1+\lambda})  - \lambda J_R(\piaux) \right]
\end{equation}

Where $\lambda$ is the Lagrange multiplier. Full derivation can be found in appendix~\ref{sec:proofs}. The resulting unconstrained optimization problem is comprised of two optimization problems. The first optimization problem (i.e., the inner loop) solves for $\pi$. Since $J_R(\piaux)$ is independent of $\pi$, this optimization is akin to solving the combined objective of Equation~\ref{eq:ea_obj} but with the importance of the imitation learning reward set to $\frac{\alpha}{1+\lambda}$. Further, for $\lambda\geq0$, we have $\alpha \geq \frac{\alpha}{1+\lambda} \geq 0$, which means that $\alpha$ is the upper bound on the importance of imitation rewards. We also refer to the importance of imitation rewards,$\frac{\alpha}{1+\lambda}$, as the \textit{balancing coefficient}. 

The second stage involves solving for $\lambda$. The dual function is always convex since it is the point-wise minimum of a linear function in $\lambda$ \citep{boyd2004convex}. Therefore it can be solved with gradient descent without worrying about local minimas. The gradient of Equation~\ref{eq:dual_prob} with respect to the lagrange multiplier, $\lambda$, leads to the following update rule:
\begin{equation}
\label{eq:update_lambda}
\lambda_{new} = \lambda_{old} - \mu [J_{R}(\pi)  - J_R(\piaux)]
\end{equation}
Where $\mu$ is the step size for updating $\lambda$. See appendix~\ref{sec:proofs} for full derivation. This update rule is intuitive: If the policy using the teacher ($\pi$) achieves more task reward than the auxiliary policy ($\piaux$) trained without the teacher, then $\lambda$ is decreased, which in turn increases$\frac{\alpha}{1+\lambda}$, making the optimization of $\pi$ more reliant on the teacher in the next iteration. Otherwise, if $\piaux$ achieves a higher reward than $\pi$, then increase in $\lambda$ decreases the importance of the teacher.

When utilizing Lagrange duality to solve a constrained optimization problem, it is necessary to consider the duality gap which is the difference between the optimal dual and primal values. A non-zero duality gap implies that the solution of the dual problem is only a lower bound to the primal problem and does not necessarily provide the exact solution~\cite{boyd2004convex}. Under certain assumptions listed in proposition~\ref{pro:dual_gap}, we show that for our optimization problem, there is \textit{no duality gap} (proof in Appendix~\ref{sec:proofs}). Thus, solving the dual problem also solves the primal problem.
\begin{proposition}
\label{pro:dual_gap}
Denote $\eta_i=J_R(\piaux^i)$. Suppose that the rewards function $r(s,a)$ and the cross-entropy term $H^X(\pi|\bar{\pi})$ are bounded. Then for every $\eta_i\in\mathbb R$ the primal and dual problems described in Eq.~\ref{eq:primal_prob} and Eq.~\ref{eq:dual_prob} have no duality gap. Moreover, if the series $\{\eta_i\}_{i=1}^\infty$ converges, then there is no duality gap in the limit.
\end{proposition}
Notice that in the general case, the cross-entropy term can reach infinity when the support of the policies does not completely overlap, violating the assumption of $H^X(\pi|\bar{\pi})$ being bounded. As a remedy, we clip the value of the cross-entropy term in our implementation of TGRL.

\begin{figure}
\begin{minipage}[t]{\linewidth}
\begin{algorithm}[H]
\begin{algorithmic}[1]
    \STATE \textbf{Input}: $\lambda_{init}$, $\alpha$, $N_{\text{collect}}$, $N_{\text{update}}$, $\mu$
    \STATE Initialize policies $\pi$ and $\pi_R$, $\lambda_0 \leftarrow \lambda_{init}$
    \FOR {$i = 1 \cdots $}
    \STATE Collect $N_{\text{collect}}$ new trajectories and add them to the replay buffer.
    \FOR {$j = 1 \cdots N_{\text{update}}$}
        \STATE Sample a batch of trajectories from the replay buffer.
        \STATE Update $Q_R$ and $Q_I$.
        \STATE Update $\pi_R$ by maximizing $Q_R$
        \STATE Update $\pi$ by maximizing $Q_R+\frac{\alpha}{1+\lambda}Q_I$
    \ENDFOR{}
    \STATE Estimate $J_R(\pi)-J_R(\piaux)$ using Eq.~\ref{eq:obj_diff_aprx}
    \STATE $\lambda_i \leftarrow \lambda_{i-1} + \mu[J_R(\pi)-J_R(\piaux)]$
    \ENDFOR{}
    \STATE \textbf{Output}: $\pi$
\end{algorithmic}
    \caption{Teacher Guided Reinforcement Learning (TGRL)} 
    
    \label{alg:alg}
\end{algorithm}
\end{minipage}
\vspace{-15pt}
\end{figure}

\subsection{Implementation}

\textbf{Off-policy approach}: We implemented our algorithm using an off-policy actor-critic approach. Off-policy learning allows data collected by both policies, $\pi$ and $\piaux$, to be stored in a common replay buffer used for training both policies. 
Our objective is to maximize the joint Q-value: $Q_{R+I}=Q_R+\frac{\alpha}{1+\lambda}Q_I$, where $Q_R, Q_I$ denote the Q-value of actions with respect to the task~(Equation~\ref{eq:env_obj}) and imitation (Equation~\ref{eq:teacher_obj}) rewards respectively. Instead of directly learning, $Q_{R+I}$, we train two critic networks, $Q_R$ and $Q_I$, and combine their values to estimate $Q_{R+I}$. This choice enables us to estimate $Q_{R+I}$ for different values of $\lambda$ without any need for re-training the critics.
We also represent $\pi$ and $\piaux$ with separate actor networks optimized to maximize the corresponding Q-values. In the data collection step, half of the trajectories are collected using $\pi$ and the other half using $\pi_R$. See Algorithm~\ref{alg:alg} for an outline of our method and Appendix~\ref{sec:exp_details} for further details.

\textbf{Estimating the performance difference}: 
As shown in Equation~\ref{eq:update_lambda}, the gradient of the dual problem with respect to $\lambda$ is the performance difference between the two policies, $J_{R}(\pi)  - J_{R}(\piaux)$.
To estimate the performance difference, one option is to perform Monte-Carlo estimation -- i.e., rollout trajectories using both policies and determine the empirical estimate of cumulative rewards. However, a good estimate of cumulative rewards requires collecting a large number of trajectories which would make our method data inefficient. Another data-efficient option is to reuse data in the replay buffer for estimating the performance difference. Because the data in the replay buffer was not collected using the current policies, we make the off-policy correction using approximations obtained by extending prior results from \citep{kakade2002approximately, schulman2015trust} known as the \textit{objective difference lemma} to the off-policy case: 
\begin{proposition} \label{pro:performance_diff}
Let $\rho(s,a,t)$ be the distribution of states, actions, and timesteps currently in the replay buffer. Then the following is an unbiased approximation of the performance difference:
$$J_{R}(\pi)  - J_{R}(\piaux)=$$
\begin{equation}
\label{eq:obj_diff_aprx}
\mathbb E_{(s,a,t)\sim \rho}{[\gamma^t(A_{\piaux}(s,a)-A_{\pi}(s,a))]}
\end{equation}
\end{proposition}
Another challenge in estimating the gradient of $\lambda$ (i.e., the performance difference between the student and the teacher policies) is the variability in the scale of the policies' performance across different environments and during training. This makes it difficult to determine an appropriate learning rate for the weighting factor $\lambda$, which will work effectively in all settings. To address this issue, we found it necessary to normalize the performance difference value during the training process. This normalization allows us to use a fixed learning rate across all of our experiments.

\section{Experiments}
\label{sec:exp}
We perform four sets of experiments. In Sec.~\ref{ssec:num1}, we provide a comparison to previous work in cases where the teacher is too good to mimic. In Sec.~\ref{ssec:num3} we solve an object reorientation problem with tactile sensors, a difficult partial observable task that both RL and TSL struggle to solve. In Sec.~\ref{ssec:num2} we look into the ability of the TGRL agent to surpass the teacher's performance. Finally, in Sec.~\ref{ssec:num4} we do ablations of our own method to show the contribution of individual components. 

\begin{figure*}[!h]
    \centering
    \includegraphics[width=\linewidth]{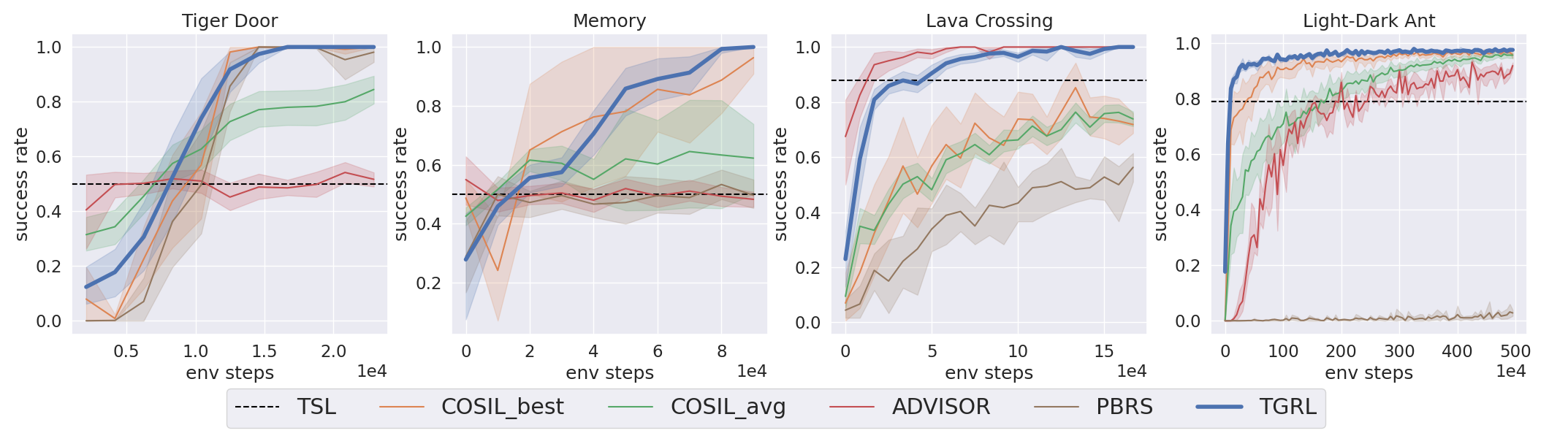}
    \vspace{-20pt}
    \caption{Comparing TGRL (blue) against algorithms proposed in prior work. TGRL is the only algorithm that performs consistently well across all environments.}
    \vspace{-10pt}
    \label{fig:main_results}
\end{figure*}

\subsection{TGRL performs well, without a need for hyperparameter tuning}  \label{ssec:num1}
We provide empirical evidence (1) showcasing the robustness of TGRL to choice of hyperparameters controlling the update of the balancing coefficient and (2) comparison to prior work.  
We compare TGRL to the following:

\textbf{TSL}. A pure Teacher-Student Learning approach that optimizes only Equation~\ref{eq:teacher_obj}.

\textbf{COSIL}~\citep{nguyen2022leveraging}. This algorithm also uses entropy-augmented RL (Eq.~\ref{eq:ea_obj}) to combine the task reward and the teacher's guidance. To adjust the balancing coefficient $\alpha$, they propose an update rule for maintaining a fixed distance ($\bar{D}$) between the student's and teacher's policies by minimizing
$\alpha(J_I(\pi)-\bar{D})$ using gradient descent. Choosing the right value of $\bar{D}$ is a challenge since its unknown apriori how similar the student and the teacher should be. Moreover, $\bar{D}$ can change drastically between environments, depending on the action space support. To tackle this issue, we run a hyperparameter search with $N=8$ different values of $\bar{D}$ and report performance for the best hyperparameter per task ($COSIL_{best}$) and average performance across hyperparameters ($COSIL_{avg}$).

\textbf{ADVISOR-off}. An off-policy version of the algorithm from \citep{weihs2021bridging} that uses a state-dependent balancing coefficient. First, an imitation policy is trained using only teacher-student learning loss. Then, for every state, the action distribution of the teacher policy is compared against the imitation policy. The states in which the two policies disagree are deemed to be ones where there is an information gap. For such states, the teacher is trusted less and more importance is given to the task reward.  

\textbf{PBRS}~\citep{walsman2023impossibly}. A potential-based reward shaping (PBRS) method based on~\citep{ng1999policy} to mitigate issues with Teacher-Student Learning. PBRS uses a given value function $V(s)$ to assign higher rewards to more beneficial states, which can lead the agent to trajectories favored by the policy associated with that value function: 
\begin{equation}
\label{eq:PBRS}
r_{new}=r_{task}+\gamma V(s_{t+1})-V(s_t)
\end{equation}
where $r_{task}$ is the task reward. Since their algorithm is on-policy, for fair comparison, we created an off-policy version of this method. For this, first, we train an imitation policy by minimizing only the teacher-student learning loss (Eq.~\ref{eq:teacher_obj}). Then, we train a neural network to represent the value function of this imitation policy. Using this value function, we obtain an augmented rewards function described in Equation~\ref{eq:PBRS}, which is then used to train a policy using the SAC algorithm~\cite{haarnoja2018soft}.

\textbf{Experimental Domains}. We experiment on diverse problems taken from prior works studying issues in Teacher-Student Learning spanning discrete and continuous action spaces, and both proprioceptive and pixel observations. For a description of each environment, see appendix~\ref{sec:exp_details}. For a fair comparison, we used the same code and Q-learning hyperparameters for all algorithms, tuning only the hyperparameters involved in balancing the teacher supervision against the task reward. The Q-learning hyperparameters correspond to hyperparameters of the RL algorithm chosen from the best-performing SAC agent. For TGRL we only used a single value for the initial coefficient $\lambda_{init}$ and coefficient learning rate $\mu$ for all tasks (more details in Appendix~\ref{sec:exp_details}). 

\begin{figure*}[t!]
    \centering
    \includegraphics[width=\linewidth]{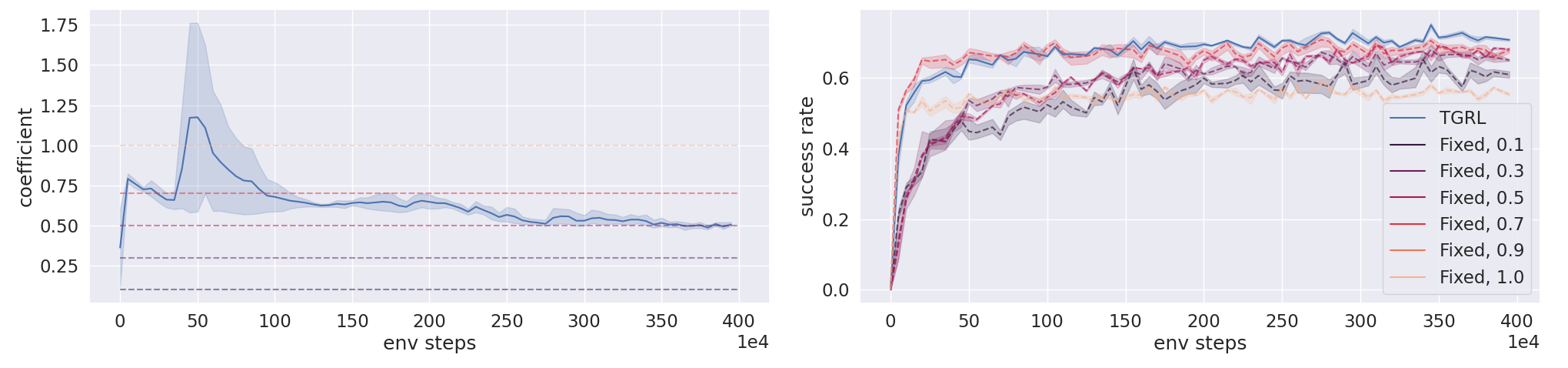}
    \vspace{-20pt}
    \caption{Adaptively balancing teacher guidance and rewards results in better asymptotic performance compared to fixing the balancing coefficient ($\lambda$). Experiment on the \textit{Shadow Hand} environment. (Left) Dynamics of $\lambda$ during training: At the start of training the agent relies more on the teacher and gradually the coefficient decreases, indicating more reliance on rewards. (Right) Performance of TGRL (blue) and ablated versions of TGRL using a fixed balancing coefficient ($\lambda$).}
    \vspace{-10pt}
    \label{fig:robustness_ablation2}
\end{figure*}

\textbf{Comparison to Baselines}. 
Results in Figure~\ref{fig:main_results} show that while each baseline method succeeds in some sub-set of tasks, no baseline method is effective on all tasks. In contrast, TGRL, solves all tasks successfully with data efficiency comparable to the best baseline in each task. Most importantly, TGRL requires no task-specific hyperparameter tuning.
Notice that $COSIL$ demonstrates comparable performance on three out of four tasks when its hyperparameters are carefully tuned (i.e., $COSIL_{best}$). However, the average performance across all hyperparameters (i.e., $COSIL_{avg}$) is significantly lower. This highlights sensitivity of $COSIL$ to the choice of hyperparameters. While \textit{PBRS} does not require hyperparameter tuning, consistent with results from another work~\citep{cheng2021heuristic}, it converges slower than other teacher-student methods and doesn't consistently perform well across tasks.

\textit{ADVISOR} achieves good performance on \textit{Lava Crossing} and \textit{Light-Dark Ant} tasks but converged to a sub-optimal policy achieving comparable performance to Teacher-Student Learning (TSL) on \textit{Tiger Door} and \textit{Memory} environments. The sub-optimal performance is due to a fundamental limitation of the \textit{ADVISOR} algorithm. As a reminder, \textit{ADVISOR} works by identifying states where the student has insufficient information to follow the teacher. For such states, instead of imitating the teacher, task rewards are used to decide the action. In the \textit{Tiger Door} environment (see Figure~\ref{fig:tiger_door}), the student has sufficient information to follow the teacher until the state at which the two arms of the environment split. However, this is too late for the student should deviate from the teacher -- to achieve optimal performance, the student should have deviated from the teacher earlier to goto the pink button. 
This example illustrates a problem that \textit{ADVISOR} encounters in environments where the information-gathering actions deviating from the teacher need to be performed before encountering the state at which the student cannot imitate the teacher.

\textbf{Robustness to Hyperparameters}.
To demonstrate the robustness of the choice of $\lambda_{init}$, we experimented with different values on \textit{Lava Crossing} environment. The results in Figure~\ref{fig:robustness_ablation} (left) shows that irrespective of the choice of $\lambda_{init}$, TGRL achieves the same asymptotic performance. This indicates that TGRL can effectively adjust $\lambda$, regardless of its initial value, $\lambda_{init}$. 

\subsection{TGRL can solve difficult environments with significant partial observability.} \label{ssec:num3}
To investigate the performance of our method on a more practical task with severe partial observability, we experimented with the Shadow hand test on task of re-orienting a pen to a target pose using only touch sensors and proprioceptive information \citep{melnik2021using}. 
Consider a Teacher-Student setup where the teacher policy observes the pen's pose and velocity. The student, however, only has access to an array of tactile sensors located on the palm of the hand and the phalanges of the fingers. To solve the task, the student needs to move his fingers along the pen and use the reading of these sensors to infer its pose. The teacher does not need to take these information-gathering actions. Thus, just mimicking the teacher will result in a sub-optimal student.

To train all agents, the reward was set to the negative of the distance between the current pen's pose and the goal. The pen has rotational symmetry around the $z$ axis, so the distance was computed only over rotations around the x and y axes. A trajectory was considered successful if the pen reached the goal orientation within 0.1 radians of the goal pose. The performance is averaged over 1,000 randomly sampled initial and goal poses.

The results are in Figure~\ref{fig:hand_manipulation}. First, to assess the difficulty of the task, we report the results of an RL agent trained with Soft Actor-Critic and Hindsight Experience Replay (HER) \citep{andrychowicz2017hindsight} over the student's observation space. This RL agent has achieved a 47\% success rate, demonstrating the difficulty of learning this task using RL alone. The teacher, trained also using SAC and HER but on the full state space, achieved a 78\% success rate. Performing vanilla Teacher-Student learning using this teacher resulted in an agent with a 54\% success rate. This performance gap shows that just imitating the teacher is not sufficient, and a deviation from the teacher's action is indeed required to learn a good policy. With TGRL, the agent achieved a significantly higher success rate of 73\%. These results demonstrate the usefulness of our algorithm and its ability to use the teacher's guidance while learning from the reward at the same time. TGRL also outperforms baseline methods (see Appendix ~\ref{sec:more_res} for more details).

\subsection{TGRL can surpass the Teacher's peformance} \label{ssec:num2}
To evaluate the ability of TGRL to surpass a sub-optimal teacher, we conducted experiments in several domains. For the \textit{Tiger Door} and \textit{Lava Crossing} environments, we constructed teachers with different optimality levels ranging from 40\% to 100\% success rate. Results in Table~\ref{table:sub_optimal} show that even with a sub-optimal teacher, TGRL learns the optimal policy in the \textit{Tiger Door} environment.
\textit{Lava Crossing} is a more challenging task, where vanilla SAC achieves 0\% success rate. Therefore, combining learning from task reward and teacher supervision allows TGRL to achieve better performance than the sub-optimal teacher, but still not 100\% success rate.

\begin{table}[h!]
\centering
\vspace{0pt}
\caption{\centering TGRL Student's success rate for sub-optimal teachers. Mean and 95\% CI over 10 random seeds. All agents were trained until convergence.}
\begin{tabular}{|P{55px}|P{44px}|P{44px}|P{44px}|}
\hline
\begin{tabular}[c]{@{}l@{}}Teacher's \\ Success Rate\end{tabular} & 100\% & 80\% & 40\% \\ \hline
\begin{tabular}[c]{@{}l@{}}Student's \\success rate - \\ Tiger Door\end{tabular}    & $100\pm0.0\%$ & $100\pm0.0\%$ & $100\pm0.0\%$ \\ \hline
\begin{tabular}[c]{@{}l@{}}Student's \\success rate - \\ Lava Crossing\end{tabular} & $100\pm0.0\%$ & $97\pm0.8\%$ & $65\pm8.1\%$ \\ \hline
\end{tabular}
\label{table:sub_optimal}
\end{table}

In addition, experimented with a variant of the Shadow Hand environment, where both student and teacher have access to the full state, but the teacher is sub-optimal. The results depicted in Figure~\ref{fig:sub_optimal} show that TGRL converges fast to the teacher's performance but than able to keep improving by utilizing task reward supervision, eventually exceeding the teacher's performance.

\begin{figure}[h!]
    \centering
    \includegraphics[width=\columnwidth]{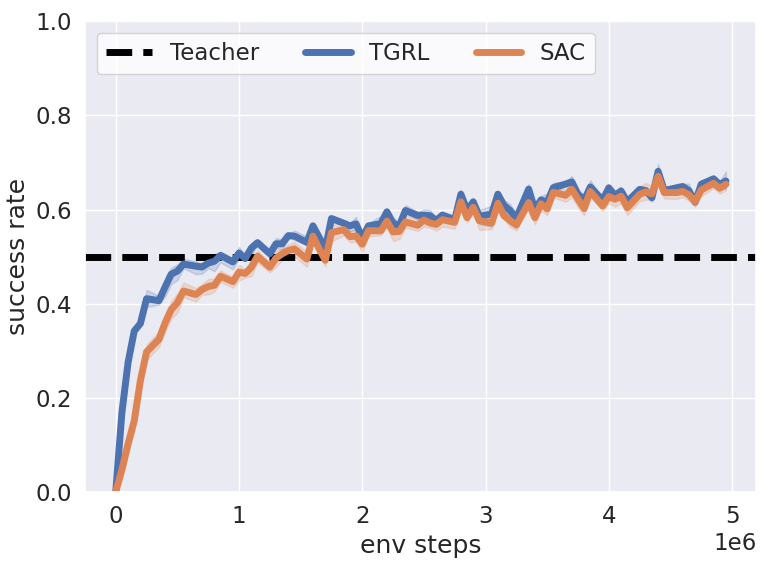}
    \caption{TGRL performance in the \textit{Shadow Hand} environment with a sub-optimal teacher. TGRL is able to surpass the teacher and achieves asymptotic performance similar to that of RL.
    }
    \label{fig:sub_optimal}
\end{figure}

\subsection{Ablations} \label{ssec:num4}
\textbf{Joint versus separate replay buffer}. 
We empirically found that having a joint replay buffer between the two policies, $\pi$ and $\piaux$, is necessary for good performance. In Figure~\ref{fig:robustness_ablation}, we compare the performance of our method with separate and joint replay buffers for the two policies on \textit{Light-Dark Ant} environment. 
As a reminder, the auxiliary policy ($\piaux$) limits the set of feasible policies in Equation~\ref{eq:const_prob}. In tasks where it is hard to learn a good policy using only task rewards, the performance of $\piaux$ will be bad leading to a loose constraint which will be ineffective. Combining the replay buffer allows $\piaux$ to learn from trajectories collected by the main policy ($\pi$), thus enabling it to achieve better performance. This, in turn, leads to a stricter constraint on the main policy, pushing it to achieve better performance.

\textbf{Fixed versus adaptive balancing coefficient}. 
A benefit of TGRL is that the balancing coefficient in the combined objective (Equation~\ref{eq:ea_obj}) dynamically changes during the training process based on the value of $\lambda$. To investigate if an adaptive coefficient is indeed beneficial, we conducted an ablation study wherein we trained policies in the \textit{Shadow Hand} environment with fixed coefficients. Figure~\ref{fig:robustness_ablation2} shows that the balancing coefficient of TGRL changes during training (left plot). At the start of training, the value is high, indicating that the teacher is given high importance. As the agent learns, the value decreases, indicating that learning from rewards is given more importance in the later stages of training. The results in the right plot of Figure~\ref{fig:robustness_ablation2} show that TGRL with a dynamically changing balancing coefficient outperforms ablated versions with fixed coefficients. 
This result indicates that TGRL goes beyond mitigating the need for searching the balancing coefficient -- it also outperforms a fixed balancing coefficient found by rigorous hyperparameter search. 

\begin{figure*}[t!]
    \centering
    \includegraphics[width=\linewidth]{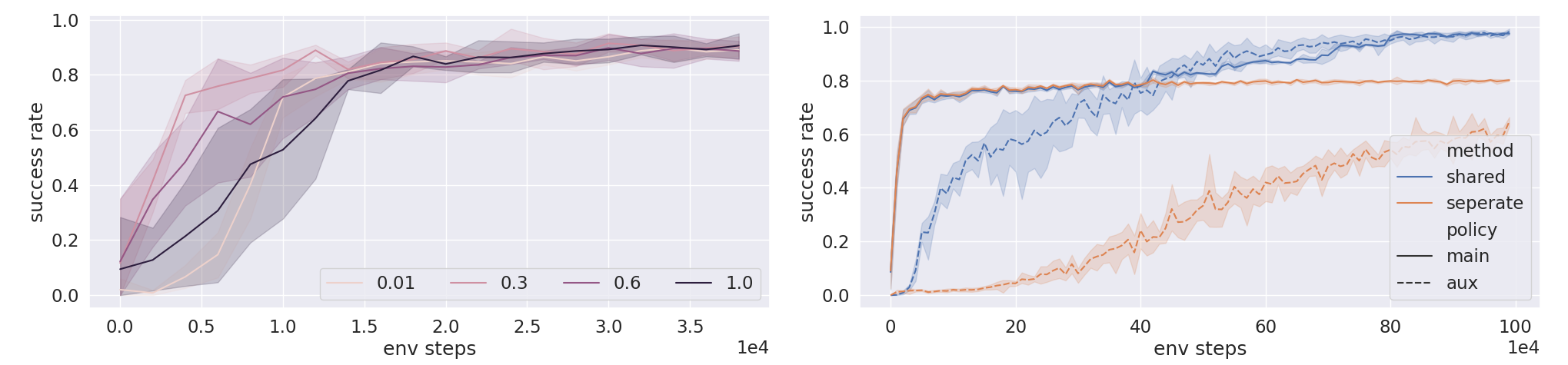}
    \vspace{-20pt}
    \caption{(Left) TGRL performance on \textit{Lava Crossing} for different values of $\lambda_{init}$. (Right) Comparing the effect of separate and joint replay buffers between the main and auxiliary policies, $\pi$ and $\pi_R$, evaluated on the \textit{Light-Dark Ant} environment.}
    \vspace{-5pt}
    \label{fig:robustness_ablation}
\end{figure*}

\section{Discussion}
\label{sec:discussion}
While TGRL improved performance across all tasks, it has its limitations.
If the agent needs to deviate substantially from the teacher, then intermediate policies during learning might have worse performance than the teacher before the agent is able to leverage rewards to improve performance. In such a case, the imitation learning policy is a local minimum, overcoming which may require additional exploration incentives~\cite{pathak2017curiosity}. 
Second, for the constraint in Equation~\ref{eq:const_prob} to be meaningful, $\piaux$ trained only with task rewards should achieve reasonable performance. While having a shared replay buffer with $\pi$ may help $\piaux$ in some hard exploration problems, learning of $\piaux$ can fail which would make the constraint ineffective.

An interesting investigation that we leave to future work is to have a state-dependent balancing coefficient. As the difference between the teacher's and student's actions can be state-dependent, such flexibility can accelerate convergence and lead to better performance. 

\section*{Acknowledgements}
We thank the members of the Improbable AI lab for the helpful discussions and feedback on the paper. We are grateful to MIT Supercloud and the Lincoln Laboratory Supercomputing Center for providing HPC resources. The research
was supported in part by the MIT-IBM Watson AI Lab, Hyundai Motor Company, DARPA Machine Common Sense Program,  and ONR MURI under Grant Number N00014-22-1-2740, and ARO MURI under Grant Number W911NF-21-1-0328.

\section*{Author Contributions}
\textbf{Idan Shenfeld} Identified the current problem with teacher-student algorithms, developed the TGRL algorithm, derived the theoretical results, conducted the experiments, and wrote the paper. 
\textbf{Zhang-Wei Hong} helped in the debugging, implementation and the choice of necessary experiments and ablations.
\textbf{Aviv Tamar} helped derive the theoretical results and provided feedback on the writing. 
\textbf{Pulkit Agrawal} oversaw the project. He was involved in the technical formulation, research
discussions, paper writing, overall advising, and positioning of the work.

\bibliography{main.bib}

\begin{thebibliography}{42}
\providecommand{\natexlab}[1]{#1}
\providecommand{\url}[1]{\texttt{#1}}
\expandafter\ifx\csname urlstyle\endcsname\relax
  \providecommand{\doi}[1]{doi: #1}\else
  \providecommand{\doi}{doi: \begingroup \urlstyle{rm}\Url}\fi

\bibitem[Madani et~al.(1999)Madani, Hanks, and
  Condon]{madani1999undecidability}
Omid Madani, Steve Hanks, and Anne Condon.
\newblock On the undecidability of probabilistic planning and infinite-horizon
  partially observable markov decision problems.
\newblock In \emph{AAAI/IAAI}, pages 541--548, 1999.

\bibitem[Papadimitriou and Tsitsiklis(1987)]{papadimitriou1987complexity}
Christos~H Papadimitriou and John~N Tsitsiklis.
\newblock The complexity of markov decision processes.
\newblock \emph{Mathematics of operations research}, 12\penalty0 (3):\penalty0
  441--450, 1987.

\bibitem[Ross et~al.(2011)Ross, Gordon, and Bagnell]{ross2011reduction}
St{\'e}phane Ross, Geoffrey Gordon, and Drew Bagnell.
\newblock A reduction of imitation learning and structured prediction to
  no-regret online learning.
\newblock In \emph{Proceedings of the fourteenth international conference on
  artificial intelligence and statistics}, pages 627--635. JMLR Workshop and
  Conference Proceedings, 2011.

\bibitem[Agarwal et~al.(2022{\natexlab{a}})Agarwal, Schwarzer, Castro,
  Courville, and Bellemare]{agarwal2022beyond}
Rishabh Agarwal, Max Schwarzer, Pablo~Samuel Castro, Aaron Courville, and
  Marc~G Bellemare.
\newblock Beyond tabula rasa: Reincarnating reinforcement learning.
\newblock \emph{arXiv preprint arXiv:2206.01626}, 2022{\natexlab{a}}.

\bibitem[Kurenkov et~al.(2019)Kurenkov, Mandlekar, Martin-Martin, Savarese, and
  Garg]{kurenkov2019ac}
Andrey Kurenkov, Ajay Mandlekar, Roberto Martin-Martin, Silvio Savarese, and
  Animesh Garg.
\newblock Ac-teach: A bayesian actor-critic method for policy learning with an
  ensemble of suboptimal teachers.
\newblock \emph{arXiv preprint arXiv:1909.04121}, 2019.

\bibitem[Rajeswaran et~al.(2017)Rajeswaran, Kumar, Gupta, Vezzani, Schulman,
  Todorov, and Levine]{rajeswaran2017learning}
Aravind Rajeswaran, Vikash Kumar, Abhishek Gupta, Giulia Vezzani, John
  Schulman, Emanuel Todorov, and Sergey Levine.
\newblock Learning complex dexterous manipulation with deep reinforcement
  learning and demonstrations.
\newblock \emph{arXiv preprint arXiv:1709.10087}, 2017.

\bibitem[Lee et~al.(2020)Lee, Hwangbo, Wellhausen, Koltun, and
  Hutter]{lee2020learning}
Joonho Lee, Jemin Hwangbo, Lorenz Wellhausen, Vladlen Koltun, and Marco Hutter.
\newblock Learning quadrupedal locomotion over challenging terrain.
\newblock \emph{Science robotics}, 5\penalty0 (47):\penalty0 eabc5986, 2020.

\bibitem[Chen et~al.(2021)Chen, Xu, and Agrawal]{chen2021system}
Tao Chen, Jie Xu, and Pulkit Agrawal.
\newblock A system for general in-hand object re-orientation.
\newblock \emph{Conference on Robot Learning}, 2021.

\bibitem[Margolis et~al.(2021)Margolis, Chen, Paigwar, Fu, Kim, Kim, and
  Agrawal]{margolis2021jumping}
Gabriel Margolis, Tao Chen, Kartik Paigwar, Xiang Fu, Donghyun Kim, Sangbae
  Kim, and Pulkit Agrawal.
\newblock Learning to jump from pixels.
\newblock \emph{Conference on Robot Learning}, 2021.

\bibitem[Zhang et~al.(2020)Zhang, Zhao, Zhao, Yin, Yang, and
  Beutel]{zhang2020deep}
Weinan Zhang, Xiangyu Zhao, Li~Zhao, Dawei Yin, Grace~Hui Yang, and Alex
  Beutel.
\newblock Deep reinforcement learning for information retrieval: Fundamentals
  and advances.
\newblock In \emph{Proceedings of the 43rd International ACM SIGIR Conference
  on Research and Development in Information Retrieval}, pages 2468--2471,
  2020.

\bibitem[Levine et~al.(2015)Levine, Finn, Darrell, and Abbeel]{levine2015end}
Sergey Levine, Chelsea Finn, Trevor Darrell, and Pieter Abbeel.
\newblock End-to-end training of deep visuomotor policies.
\newblock \emph{arXiv preprint arXiv:1504.00702}, 2015.

\bibitem[Kumor et~al.(2021)Kumor, Zhang, and Bareinboim]{kumor2021sequential}
Daniel Kumor, Junzhe Zhang, and Elias Bareinboim.
\newblock Sequential causal imitation learning with unobserved confounders.
\newblock \emph{Advances in Neural Information Processing Systems},
  34:\penalty0 14669--14680, 2021.

\bibitem[Swamy et~al.(2022)Swamy, Choudhury, Bagnell, and
  Wu]{swamy2022sequence}
Gokul Swamy, Sanjiban Choudhury, J~Andrew Bagnell, and Zhiwei~Steven Wu.
\newblock Sequence model imitation learning with unobserved contexts.
\newblock \emph{arXiv preprint arXiv:2208.02225}, 2022.

\bibitem[Littman et~al.(1995)Littman, Cassandra, and
  Kaelbling]{littman1995learning}
Michael~L Littman, Anthony~R Cassandra, and Leslie~Pack Kaelbling.
\newblock Learning policies for partially observable environments: Scaling up.
\newblock In \emph{Machine Learning Proceedings 1995}, pages 362--370.
  Elsevier, 1995.

\bibitem[Warrington et~al.(2021)Warrington, Lavington, Scibior, Schmidt, and
  Wood]{warrington2021robust}
Andrew Warrington, Jonathan~W Lavington, Adam Scibior, Mark Schmidt, and Frank
  Wood.
\newblock Robust asymmetric learning in pomdps.
\newblock In \emph{International Conference on Machine Learning}, pages
  11013--11023. PMLR, 2021.

\bibitem[Weihs et~al.(2021)Weihs, Jain, Liu, Salvador, Lazebnik, Kembhavi, and
  Schwing]{weihs2021bridging}
Luca Weihs, Unnat Jain, Iou-Jen Liu, Jordi Salvador, Svetlana Lazebnik,
  Aniruddha Kembhavi, and Alex Schwing.
\newblock Bridging the imitation gap by adaptive insubordination.
\newblock \emph{Advances in Neural Information Processing Systems},
  34:\penalty0 19134--19146, 2021.

\bibitem[Nguyen et~al.(2022)Nguyen, Baisero, Wang, Amato, and
  Platt]{nguyen2022leveraging}
Hai Nguyen, Andrea Baisero, Dian Wang, Christopher Amato, and Robert Platt.
\newblock Leveraging fully observable policies for learning under partial
  observability.
\newblock \emph{arXiv preprint arXiv:2211.01991}, 2022.

\bibitem[Agarwal et~al.(2022{\natexlab{b}})Agarwal, Schwarzer, Castro,
  Courville, and Bellemare]{agarwal2022reincarnating}
Rishabh Agarwal, Max Schwarzer, Pablo~Samuel Castro, Aaron Courville, and
  Marc~G Bellemare.
\newblock Reincarnating reinforcement learning: Reusing prior computation to
  accelerate progress.
\newblock \emph{arXiv preprint arXiv:2206.01626}, 2022{\natexlab{b}}.

\bibitem[Kaelbling et~al.(1998)Kaelbling, Littman, and
  Cassandra]{kaelbling1998planning}
Leslie~Pack Kaelbling, Michael~L Littman, and Anthony~R Cassandra.
\newblock Planning and acting in partially observable stochastic domains.
\newblock \emph{Artificial intelligence}, 101\penalty0 (1-2):\penalty0 99--134,
  1998.

\bibitem[Sutton and Barto(2018)]{sutton2018reinforcement}
Richard~S Sutton and Andrew~G Barto.
\newblock \emph{Reinforcement learning: An introduction}.
\newblock 2018.

\bibitem[Argall et~al.(2009)Argall, Chernova, Veloso, and
  Browning]{argall2009survey}
Brenna~D Argall, Sonia Chernova, Manuela Veloso, and Brett Browning.
\newblock A survey of robot learning from demonstration.
\newblock \emph{Robotics and autonomous systems}, 2009.

\bibitem[Czarnecki et~al.(2019)Czarnecki, Pascanu, Osindero, Jayakumar,
  Swirszcz, and Jaderberg]{czarnecki2019distilling}
Wojciech~M Czarnecki, Razvan Pascanu, Simon Osindero, Siddhant Jayakumar,
  Grzegorz Swirszcz, and Max Jaderberg.
\newblock Distilling policy distillation.
\newblock In \emph{The 22nd International Conference on Artificial Intelligence
  and Statistics}, pages 1331--1340. PMLR, 2019.

\bibitem[Ng et~al.(1999)Ng, Harada, and Russell]{ng1999policy}
Andrew~Y Ng, Daishi Harada, and Stuart Russell.
\newblock Policy invariance under reward transformations: Theory and
  application to reward shaping.
\newblock In \emph{Icml}, volume~99, pages 278--287, 1999.

\bibitem[Schmitt et~al.(2018)Schmitt, Hudson, Zidek, Osindero, Doersch,
  Czarnecki, Leibo, Kuttler, Zisserman, Simonyan,
  et~al.]{schmitt2018kickstarting}
Simon Schmitt, Jonathan~J Hudson, Augustin Zidek, Simon Osindero, Carl Doersch,
  Wojciech~M Czarnecki, Joel~Z Leibo, Heinrich Kuttler, Andrew Zisserman, Karen
  Simonyan, et~al.
\newblock Kickstarting deep reinforcement learning.
\newblock \emph{arXiv preprint arXiv:1803.03835}, 2018.

\bibitem[Chen et~al.(2022)Chen, Hong, Pajarinen, and
  Agrawal]{chen2022redeeming}
Eric Chen, Zhang-Wei Hong, Joni Pajarinen, and Pulkit Agrawal.
\newblock Redeeming intrinsic rewards via constrained optimization.
\newblock \emph{arXiv preprint arXiv:2211.07627}, 2022.

\bibitem[Tessler et~al.(2018)Tessler, Mankowitz, and Mannor]{tessler2018reward}
Chen Tessler, Daniel~J Mankowitz, and Shie Mannor.
\newblock Reward constrained policy optimization.
\newblock \emph{arXiv preprint arXiv:1805.11074}, 2018.

\bibitem[Bhatnagar and Lakshmanan(2012)]{bhatnagar2012online}
Shalabh Bhatnagar and K~Lakshmanan.
\newblock An online actor--critic algorithm with function approximation for
  constrained markov decision processes.
\newblock \emph{Journal of Optimization Theory and Applications}, 153\penalty0
  (3):\penalty0 688--708, 2012.

\bibitem[Boyd et~al.(2004)Boyd, Boyd, and Vandenberghe]{boyd2004convex}
Stephen Boyd, Stephen~P Boyd, and Lieven Vandenberghe.
\newblock \emph{Convex optimization}.
\newblock Cambridge university press, 2004.

\bibitem[Kakade and Langford(2002)]{kakade2002approximately}
Sham Kakade and John Langford.
\newblock Approximately optimal approximate reinforcement learning.
\newblock In \emph{In Proc. 19th International Conference on Machine Learning}.
  Citeseer, 2002.

\bibitem[Schulman et~al.(2015)Schulman, Levine, Abbeel, Jordan, and
  Moritz]{schulman2015trust}
John Schulman, Sergey Levine, Pieter Abbeel, Michael Jordan, and Philipp
  Moritz.
\newblock Trust region policy optimization.
\newblock In \emph{Proceedings of the 32nd International Conference on Machine
  Learning (ICML-15)}, pages 1889--1897, 2015.

\bibitem[Walsman et~al.(2023)Walsman, Zhang, Choudhury, Fox, and
  Farhadi]{walsman2023impossibly}
Aaron Walsman, Muru Zhang, Sanjiban Choudhury, Dieter Fox, and Ali Farhadi.
\newblock Impossibly good experts and how to follow them.
\newblock In \emph{The Eleventh International Conference on Learning
  Representations}, 2023.

\bibitem[Haarnoja et~al.(2018)Haarnoja, Zhou, Abbeel, and
  Levine]{haarnoja2018soft}
Tuomas Haarnoja, Aurick Zhou, Pieter Abbeel, and Sergey Levine.
\newblock Soft actor-critic: Off-policy maximum entropy deep reinforcement
  learning with a stochastic actor.
\newblock \emph{arXiv preprint arXiv:1801.01290}, 2018.

\bibitem[Cheng et~al.(2021)Cheng, Kolobov, and Swaminathan]{cheng2021heuristic}
Ching-An Cheng, Andrey Kolobov, and Adith Swaminathan.
\newblock Heuristic-guided reinforcement learning.
\newblock \emph{Advances in Neural Information Processing Systems},
  34:\penalty0 13550--13563, 2021.

\bibitem[Melnik et~al.(2021)Melnik, Lach, Plappert, Korthals, Haschke, and
  Ritter]{melnik2021using}
Andrew Melnik, Luca Lach, Matthias Plappert, Timo Korthals, Robert Haschke, and
  Helge Ritter.
\newblock Using tactile sensing to improve the sample efficiency and
  performance of deep deterministic policy gradients for simulated in-hand
  manipulation tasks.
\newblock \emph{Frontiers in Robotics and AI}, page~57, 2021.

\bibitem[Andrychowicz et~al.(2017)Andrychowicz, Wolski, Ray, Schneider, Fong,
  Welinder, McGrew, Tobin, Abbeel, and Zaremba]{andrychowicz2017hindsight}
Marcin Andrychowicz, Filip Wolski, Alex Ray, Jonas Schneider, Rachel Fong,
  Peter Welinder, Bob McGrew, Josh Tobin, OpenAI~Pieter Abbeel, and Wojciech
  Zaremba.
\newblock Hindsight experience replay.
\newblock In \emph{Advances in Neural Information Processing Systems}, pages
  5048--5058, 2017.

\bibitem[Pathak et~al.(2017)Pathak, Agrawal, Efros, and
  Darrell]{pathak2017curiosity}
Deepak Pathak, Pulkit Agrawal, Alexei~A Efros, and Trevor Darrell.
\newblock Curiosity-driven exploration by self-supervised prediction.
\newblock In \emph{Proceedings of the 34th International Conference on Machine
  Learning}, pages 2778--2787, 2017.

\bibitem[Paternain et~al.(2019)Paternain, Chamon, Calvo-Fullana, and
  Ribeiro]{paternain2019constrained}
Santiago Paternain, Luiz Chamon, Miguel Calvo-Fullana, and Alejandro Ribeiro.
\newblock Constrained reinforcement learning has zero duality gap.
\newblock \emph{Advances in Neural Information Processing Systems}, 32, 2019.

\bibitem[Rockafellar(1970)]{rockafellar1970convex}
R~Tyrrell Rockafellar.
\newblock \emph{Convex analysis}, volume~18.
\newblock Princeton university press, 1970.

\bibitem[Platt~Jr et~al.(2010)Platt~Jr, Tedrake, Kaelbling, and
  Lozano-Perez]{platt2010belief}
Robert Platt~Jr, Russ Tedrake, Leslie Kaelbling, and Tomas Lozano-Perez.
\newblock Belief space planning assuming maximum likelihood observations.
\newblock 2010.

\bibitem[Mnih et~al.(2015)Mnih, Kavukcuoglu, Silver, Rusu, Veness, Bellemare,
  Graves, Riedmiller, Fidjeland, Ostrovski, Petersen, Beattie, Sadik,
  Antonoglou, King, Kumaran, Wierstra, Legg, and Hassabis]{mnih2015human}
Volodymyr Mnih, Koray Kavukcuoglu, David Silver, Andrei~A Rusu, Joel Veness,
  Marc~G Bellemare, Alex Graves, Martin Riedmiller, Andreas~K Fidjeland, Georg
  Ostrovski, Stig Petersen, Charles Beattie, Amir Sadik, Ioannis Antonoglou,
  Helen King, Dharshan Kumaran, Daan Wierstra, Shane Legg, and Demis Hassabis.
\newblock Human-level control through deep reinforcement learning.
\newblock \emph{Nature}, 2015.

\bibitem[Lillicrap et~al.(2015)Lillicrap, Hunt, Pritzel, Heess, Erez, Tassa,
  Silver, and Wierstra]{lillicrap2015continuous}
Timothy~P Lillicrap, Jonathan~J Hunt, Alexander Pritzel, Nicolas Heess, Tom
  Erez, Yuval Tassa, David Silver, and Daan Wierstra.
\newblock Continuous control with deep reinforcement learning.
\newblock \emph{arXiv preprint arXiv:1509.02971}, 2015.

\bibitem[Ni et~al.(2022)Ni, Eysenbach, and Salakhutdinov]{ni2022recurrent}
Tianwei Ni, Benjamin Eysenbach, and Ruslan Salakhutdinov.
\newblock Recurrent model-free rl can be a strong baseline for many pomdps.
\newblock In \emph{International Conference on Machine Learning}, pages
  16691--16723. PMLR, 2022.

\end{thebibliography}

\newpage

\appendix

\section{Derivations and Proofs} \label{sec:proofs}
\subsection{Derivation of the Dual Problem} \label{ssec:full_derivation}
Denote $\eta_i=J_R(\pi^i_{RL})$, and given the Primal Problem we derived in Eq.~\ref{eq:primal_prob}:
$$
\max _{\pi}{J_{R+I}(\pi, \alpha)} \quad \textrm{subject to} \quad J_R(\pi) \geq \eta_i
$$
The corresponding Lagrangian is:
$$
\mathcal{L}(\pi, \lambda) = 
J_{R+I}(\pi, \alpha) +\lambda \left( J_R(\pi) - \eta_i\right) =
$$
$$
 \mathbb E_{\pi} \left[ \sum_{t=0}^\infty {\gamma^t(r_t-\alpha H^X_t(\pi|\bar{\pi}))}\right] + \lambda  \mathbb E_{\pi} \left[ \sum_{t=0}^\infty {\gamma^tr_t}\right] - \lambda \eta_i =
$$
$$
\mathbb E_{\pi} \left[ \sum_{t=0}^\infty {\gamma^t \left( (1+\lambda)r_t-\alpha H^X_t(\pi|\bar{\pi}) \right) }\right]  - \lambda \eta_i =
$$
$$
\mathbb E_{\pi} \left[ (1+\lambda)\sum_{t=0}^\infty {\gamma^t\left(r_t-\frac{\alpha}{1+\lambda}H^X_t(\pi|\bar{\pi})\right)}\right]  - \lambda \eta_i  = 
$$
$$
(1+\lambda) J_{R+I}(\pi, \frac{\alpha}{1+\lambda})  - \lambda \eta_i 
$$
And therefore out Dual problem is:
$$
\min_{\lambda \geq 0}{\max_{\pi}} \left[ (1+\lambda) J_{R+I}(\pi, \frac{\alpha}{1+\lambda})  - \lambda \eta_i \right]
$$

\subsection{Derivation of update rule for $\lambda$}
The gradient of the dual problem with respect to $\lambda$ is:
$$
\nabla_\lambda \left[ (1+\lambda) J_{R+I}(\pi, \frac{\alpha}{1+\lambda})  - \lambda \eta_i \right] = 
$$
$$
\nabla_\lambda \left[\mathbb E_{\pi} \left[ \sum_{t=0}^\infty {\gamma^t \left( (1+\lambda)r_t-\alpha H^X_t(\pi|\bar{\pi}) \right) }\right]  - \lambda \eta_i \right] =
$$
$$
\mathbb E_{\pi} \left[ \sum_{t=0}^\infty {\gamma^t r_t }\right]  - \eta_i =
$$
$$
J_{R}(\pi)  - \eta_i
$$

\subsection{Duality Gap - Proof for Proposition~\ref{pro:dual_gap}}
We start by restating our assumptions and discuss why they hold for our problem:
\begin{assumption} \label{ass:bound}
The rewards function $r(s,a)$ and the cross-entropy term $H^X(\pi|\bar{\pi})$ are bounded.
\end{assumption}
\textbf{Justification for A.1.} This is achieved by using a clipped version of the cross entropy term. We will add that we found the clipping helpful in practice since it stops this term from reaching infinity when the support of the teacher and the student action distributions are not the same.
\begin{assumption} \label{ass:convergence}
The sequence $\{\eta_i\}_{i=1}^\infty$ is monotincally increasing and converging, i.e., there exist $\eta\in\mathbb R$ such that $\lim_{i\rightarrow\infty} \eta_i = \eta$.
\end{assumption}
\textbf{Justification for A.2.} We will remind that the sequence $\{\eta_i\}_{i=1}^\infty$ is the result of incrementally solving $\max _{\pi_R} {J_R(\pi_R)}$. Having this sequence be monotonically increasing is equivalent to a guarantee for policy improvement in each optimization step, an attribute of several RL algorithms such as Q-learning or policy gradient \citep{sutton2018reinforcement}. Regarding convergence, since the reward is upper bound from assumption~\ref{ass:bound}, then we have an upper bounded monotonically increasing sequence of real numbers, which is proved to converge.
\begin{assumption} \label{ass:epsilon-bound}
There exist $\epsilon>0$ such that for all $i$, the value of $\eta_i$ is at most $J_R(\pi^*)-\epsilon$.
\end{assumption}
\textbf{Justification for A.3.} This assumption is equivelant to stating that $J_R(\pi^*)-J_R(\pi_R)>0$, meaning that $\pi_R$ is never optimal. Without further assumption on the algorithm used to optimize $\pi_R$, we can not guarantee that this will not happen. However, if it happens, it means that we were able to find the optimal policy, and therefore there is no need to continue with the optimization procedure. As a remedy, we will define a new sequence $\{\tilde{\eta}_i\}_{i=1}^\infty$ where $\tilde{\eta}_i = \eta_i-\epsilon$ and will use it instead of the original $\eta_i$. Since $\epsilon$ can be as small as we want, its effect on the algorithm is negligible and it served mainly for the completeness of our theory.

Before going into our proof, we will cite Theorem 1 of \citep{paternain2019constrained}, which is the basis of our results:
\begin{theorem} \label{theorem:main_res}
Given the following optimization problem:
$$
P^* = \max_\pi {\mathbb E_\pi \left[ \sum_{i=0}^H {\gamma^t r_0(s_t,a_t)} \right]} \quad \text{subject to} \quad $$$$
\mathbb E_\pi \left[ \sum_{i=0}^H {\gamma^t r_i(s_t,a_t)} \right] \geq c_i, \quad i=1...m,
$$
And its Dual form:
$$
D^* = \min_{\lambda \geq 0} \max_\pi \mathbb E_\pi \left[ \sum_{i=0}^H {\gamma^t r_0(s_t,a_t)} \right] + 
$$
$$
\lambda \sum_{i=1}^m{\left[ \mathbb E_\pi \left[ \sum_{i=0}^H {\gamma^t r_i(s_t,a_t)} \right] - c_i\right] }
$$
suppose that $r_i$ is bounded for all $i = 0, . . . , m$ and that Slater’s condition holds. Then, strong duality holds, i.e., $P^* = D^*$.
\end{theorem}
Having stated that, we will move to prove the two parts of our proposition:
\begin{proposition} \label{pro:part_1}
Given assumption~\ref{ass:bound} and~\ref{ass:epsilon-bound}, for every $\eta_i\in\mathbb R$, the constrained optimization problem Eq.~\ref{eq:primal_prob} and its dual problem defined in Eq.~\ref{eq:dual_prob} do not have a duality gap.
\end{proposition}
\begin{proof}
We align our problem with Theorem~\ref{theorem:main_res} notations by denote as follows:

$$
r_0:r_t-\alpha H^X_t, \quad r_1:r_t, \quad c_1:\eta_i
$$

And we can see that our problem is a specific case of the optimization problem defined above. For every $\eta_i$, there is a set feasible solutions in the form of an $\epsilon$-neighborhood of $\pi^*$. This holds since $J_R(\pi^*)>J_R(\pi)-\epsilon$ for every $\pi\notin\pi^*$. Therefore, Slater’s condition holds as it required that the feasible solution set will have an interior point. Together with assumption~\ref{ass:bound}, we have all that we need to claim that Theorem~\ref{theorem:main_res} applies to our problem. Therefore, there is no duality gap.
\end{proof}

\begin{proposition}
Given all our assumptions, the constrained optimization problem at the limit:
$$
\max _{\pi}{J_{R+I}(\pi, \alpha)} \quad \textrm{subject to} \quad J_R(\pi) \geq \eta
$$
has no duality gap.
\end{proposition}
\begin{proof}
Our proof will be based on the Fenchel-Moreau theorem \citep{rockafellar1970convex} that states: 

\textit{If (i) Slater’s condition holds for the primal problem and (ii) its perturbation function P($\xi$) is concave, then strong duality holds for the primal and dual problems.}

Denote $\eta_{\lim}$ the limit of the sequence. Without loss of generality, we assume that $\eta_{\lim}=J_R(\pi^*)-\epsilon$. If not, we will just adjust $\epsilon$ accordingly. As in the last proof, Slater’s condition holds since there is a set of feasible policies for the problem. Regarding the second requirement, the sequence of perturbation functions for our problem is:
$$
P(\xi)=\lim _{i\rightarrow \infty } P_i(\xi)
$$ $$ \text{where} \quad P_i(\xi)=
\max _{\pi}{J_{R+I}(\pi, \alpha)} \quad $$ $$ \textrm{subject to} \quad J_R(\pi)  \geq \eta_i+\xi
$$

Notice that this is a scalar function since $P_i(\xi)$ is the maximum objective itself, not the policy that induces it. We will now prove that this sequence of functions converges point-wise:
\begin{itemize}
\item For all $\xi>\epsilon$ we claim that  $P(\xi)=\lim _{i\rightarrow \infty } P_i(\xi)=-\infty$. As a reminder $\eta_i$ converged to $J_R(\pi^*)-\epsilon$. It means that there exists $N$ such that for all $n>N$ , we have $|\eta_n-J_R(\pi^*)+\epsilon|<\frac{\xi}{2}-\epsilon$. Moreover, since $J_R(\pi^*)-\epsilon$ is also the upper bound on the series of $\eta_i$ we can remove the absolute value and get:
$$
0\leq J_R(\pi^*)-\epsilon-\eta_n<\frac{\xi}{2}-\epsilon$$
This yields the following constraint:
    $$
    J_R(\pi_\theta)  \geq \eta_n+\xi > J_R(\pi^*) - \frac{\xi}{2}+\xi = J_R(\pi^*) + \frac{\xi}{2}
    $$
    \normalsize
But since $\xi>\epsilon>0$ and $\pi^*$ is the optimal policy, no policies are feasible for this constraint, so from the definition of the perturbation function, we have $P_n(\xi)=-\infty$. This holds for all $n>N$ and, therefore also $\lim _{i\rightarrow \infty } P_i(\xi)=-\infty$.

\item For all $\xi \leq \epsilon$ we will prove convergence to a fixed value. First, we claim that the perturbation function has a lower bound. This is true since the reward function and the cross-entropy are bounded, and the perturbation function value is a discounted sum of them.\\
In addition, the sequence of $P_i(\xi)$ is monotonically decreasing. To see it, remember that the sequence $\{\eta_i\}_{i=1}^\infty$ is monotonically increasing. Since $J_R(\pi)$ is also upper bounded by $J_R(\pi^*)$, then the feasible set of the $(i+1)$ problem is a subset of the feasible set of the $i$ problem, and all those which came before. Therefore if the solution to the $i$ problem is still feasible it will also be the solution to the $i+1$ problem. If not, then it has a lower objective (since it was also feasible in the $i$ problem), resulting in a monotonically decreasing sequence. Finally, for every $\eta_i$ there is at least one feasible solution, $J_R(\pi^*)$, meaning the perturbation function has a real value. To conclude, $\{P_i(\xi)\}_{i=1}^\infty$ is a monotonically decreasing, lower-bounded sequence in $\mathbb R$ in therefore it converged.
\end{itemize}

After we established point-wise convergence to a function $P(\xi)$, all that remain is to proof that this function is concave. According to proposition~\ref{pro:part_1}, each optimization problem doesn’t have a duality gap, meaning its perturbation function is concave. Since every function in the sequence is concave, and there is pointwise convergence, $P(\xi)$ is also concave. To conclude, from the Fenchel-Moreau theorem, our optimization problem doesn’t have a duality gap in the limit.
\end{proof}

\subsection{Performance Difference Estimation - Proof for Proposition 3}

\textit{Proposition:} Let $\rho(s,a,t)$ be the distribution of states, actions, and timesteps currently in the replay buffer. Then the following is an unbiased approximation of the performance difference:
$$J_{R}(\pi)  - J_{R}(\pi_R)=$$
$$
\mathbb E_{(s,a,t)\sim \rho}{[\gamma^t(A_{\pi_R}(s,a)-A_{\pi}(s,a))]}
$$
\textit{Proof:} Let $\pi_{RB}$ be the behavioral policy induced by the data currently in the replay buffer, meaning:

$$
\forall s\in S \quad \pi_{RB}(a|s)=\frac{\sum_{a'\in {RB}(s)}{I_{a'=a}}}{\sum_{a'\in {RB}(s)}{1}}
$$

Using lemma 6.1 from \citep{kakade2002approximately}, for every two policies $\pi$ and $\tilde{\pi}$ We can write:
$$
\eta(\tilde{\pi})-\eta(\pi)
=\eta(\tilde{\pi})-\eta(\pi_{RB})+\eta(\pi_{RB})-\eta(\pi)=$$
$$-[\eta(\pi_{RB})-\eta(\tilde{\pi})]+\eta(\pi_{RB})-\eta(\pi)=$$
$$-\sum_s{\sum_{t=0}^\infty{\gamma^tP(s_t=s|\pi_{RB})\sum_a{\pi_{RB}(a|s)A_{\tilde{\pi}}(s,a)}}}+$$
$$\sum_s{\sum_{t=0}^\infty{\gamma^tP(s_t=s|\pi_{RB})\sum_a{\pi_{RB}(a|s)A_{{\pi}}(s,a)}}}=$$
\small$$\sum_s{\sum_{t=0}^\infty{\gamma^tP(s_t=s|\pi_{RB})\sum_a{\pi_{RB}(a|s)[A_{{\pi}}(s,a)-A_{\tilde{\pi}}(s,a)]}}}
$$
\normalsize
Assuming we can sample tuples of $(s, a, t)$ from our replay buffer and denote this distribution $\rho_{RB}(s,a,t)$ we can write the above equation as:

$$
\eta(\tilde{\pi})-\eta(\pi)=\sum_{s,a,t}{\rho_{RB}(s,a,t)\gamma^t[A_{{\pi}}(s,a)-A_{\tilde{\pi}}(s,a)]}
$$

Which we can approximate by sampling such tuples from the replay buffer.

\section{Experimental Details} \label{sec:exp_details}

In this section, we outline our environment, training process and hyperparameters.

\textbf{Environment details.} The following list contain details about all the environment used to test our algorithm and compare it to the baselines.

\textit{Tiger Door}. A robot must navigate to the goal cell (green), without touching the failure cell (blue). The cells, however, randomly switch locations every episode, and their nature is not observed by the agent. The maze also includes a pink button that reveals the correct goal location. Pixel observations with discrete action space.

\textit{Lava Crossing}. A minigrid environment where the agent starts in the top-left corner and needs to navigate through a maze of lava in order to get to the bottom-right corner. The episode ends in failure if the agent steps on the lava. The teacher has access to the whole map, while the student only sees a patch of 5x5 cells in front of it. Pixel observations with discrete action space.

\textit{Memory}.  A minigrid environment. The agent starts in a corridor containing two objects. It then has to go to a nearby room containing a third object, similar to one of the previous two. The agent's goal is to go back and touch the object it saw in the room. The episode ends in success if the agent goes to the correct object and in failure otherwise. While the student has to go to the room to see which object is the current one, the teacher starts with that knowledge and can go to it directly. Pixel observations with discrete action space.

\textit{Light-Dark Ant}. A Mujoco Ant environment with a fixed goal and a random starting position. The starting position and the goal are located at the "dark" side of the room, where the agent has access only to a noisy measurement of its current location. It has to take a detour through the "lighted" side of the room, where the noise is reduced significantly, enabling it to understand its location. On the other hand, the teacher has access to its precise location at all times, enabling it to go directly to the goal. This environment is inspired by a popular POMDP benchmark ~\citep{platt2010belief}. Proprioceptive observation with continuous action space.

\textbf{Training process. }Our algorithm optimizes two policies, $\pi$, and $\pi_R$, using off-policy Q-learning. The algorithm itself is orthogonal to the exact details of how to perform this optimization. For the discrete Gridworld domains (\textit{Tiger Door}, \textit{Memory} and \textit{Lava Crossing}), we used DQN \citep{mnih2015human} with soft target network updates, as proposed by \citep{lillicrap2015continuous}, which has shown to improve the stability of learning. For the rest of the continuous domains, we used SAC \citep{haarnoja2018soft} with the architectures of the actor and critic chosen similarly and with a fixed entropy coefficient. For both DQN and SAC, we set the soft target update parameter to 0.005. As was mentioned in the paper, we represent the Q function using to separate networks, one for estimating $Q_R$ and another for estimating $Q_E$. When updating a Q function, it has to be done with respect to some policy. We found that doing so with respect to policy $\pi$ yields stable performance across all environments.

For \textit{Tiger Door}, \textit{Memory}, and \textit{Lava Crossing}, the teacher is a shortest-path algorithm executed over the grid map. For \textit{Light-Dark Ant}, the teacher is a policy trained using RL over the teacher's observation space until achieving a success rate of 100\%. In all of our experiments, we average performance over 5 random seeds and present the mean and 95\% confidence interval.

For all proprioceptive domains, we used a similar architecture across all algorithms. The architecture includes two fully-connected (FC) layers for embedding the past observations and actions separately. These embeddings are then passed through a Long Short-Term Memory (LSTM) layer to aggregate the inputs across the temporal domain. Additionally, the current observation is embedded using an FC layer and concatenated with the output of the LSTM. The concatenated representation is then passed through another fully-connected network with two hidden layers, which outputs the action. The architecture for pixel-based observations are the same, with the observations encoded by a Convolutional Neural Network (CNN) instead of FC. The number of neurons in each layer is determined by the specific domain. 
The rest of the hyperparameters used for training the agents are summarized in~\ref{fig:hyper_table}.

Our implementation is based on the code released by \citep{ni2022recurrent}.

\textbf{Fair Hyperparameter Tuning.} We attempt to ensure that comparisons to baselines are fair. In particular, as part of our claim that our algorithm is more robust to the choice of its hyperparameters, we took the following steps. First, we re-implemented all baselines, and while conducting experiments, maintained consistent joint hyperparameters across the various algorithms. Second, all the experiments of our own algorithm, TGRL, used the same hyperparameters. We used $\alpha=3$, initial $\lambda$ equal to 9 (and so the effective coefficient $\frac{\alpha}{1+\lambda}=0.3$) and coefficient learning rate of $3\mathrm{e}{-3}$. Finally, for every one of the baselines we performed for each environment a search over all the algorithm-specific hyperparameters with N=8 different values for each one and report the best results (besides for COSIL, where we also report the average performance across hyperparameters). 

\section{Additional Results} \label{sec:more_res}
Here we record additional results that were summarized or deferred in Section 4. In particular:

\textbf{Environments without information differences}. Determining if the information difference between the teacher and the student in a given environment will lead to a sub-optimal student is a complex task, as it is dependent on the specific task and the observations available to the agent, which can vary significantly across different environments. As such, it can be challenging to know beforehand if this problem exists or not. In the following experiment, we demonstrate that even in scenarios where this problem does not exist, the use of our proposed TGRL algorithm yields results that are comparable to those obtained using traditional Teacher-Student Learning (TSL) methods, which are typically considered the best approach in such scenarios. This highlights the robustness and versatility of our proposed approach.

The experiment includes three classic POMDP environments from \citep{ni2022recurrent}. These environments are a version of the Mujoco \textit{Hopper}, \textit{Walker2D}, and \textit{HalfCheetah} environments, where the agent only have access to the joint positions but not to their velocities. The teacher, however, has access to both positions and velocities. As can be seen in Figure~\ref{fig:no_imitation_gap}, TGRL converges a bit slower than TSL but still manage to converge to the teacher's performance.

\textbf{Full training curves for Shadow Hand experiments.} In Figure~\ref{fig:full_shadow_hand}, we provide the full version of the training curves that appears in Figure~\ref{fig:hand_manipulation}.

\begin{figure*}[]
\centering
\begin{tabular}{|l|cllcc|}
\hline
                                        & \multicolumn{1}{l|}{\textbf{Tiger Door}} & \multicolumn{1}{l|}{\textbf{Lava Crossing}} & \multicolumn{1}{l|}{\textbf{Memory}} & \multicolumn{1}{l|}{\textbf{Light-Dark Ant}} & \multicolumn{1}{l|}{\textbf{Shadow Hand}} \\ \hline
Max ep. length                          & \multicolumn{1}{c|}{100}                 & \multicolumn{1}{c|}{225}                    & \multicolumn{1}{c|}{121}             & \multicolumn{1}{c|}{100}                     & 100                                       \\ \cline{2-4}
Collected ep. per iter.                 & \multicolumn{3}{c|}{5}                                                                                                        & \multicolumn{1}{c|}{10}                      & 120                                       \\
RL updates per iter.                    & \multicolumn{3}{c|}{500}                                                                                                      & \multicolumn{1}{c|}{1000}                    & 1000                                      \\ \cline{2-6} 
Optimizer                               & \multicolumn{5}{c|}{Adam}                                                                                                                                                                                                \\
Learning rate                           & \multicolumn{5}{c|}{3e-4}                                                                                                                                                                                                \\
Discount factor ($\gamma$) & \multicolumn{5}{c|}{0.9}                                                                                                                                                                                                 \\ \cline{2-6} 
Batch size                              & \multicolumn{3}{c|}{32}                                                                                                       & \multicolumn{1}{c|}{128}                     & 128                                       \\
LSTM hidden size                        & \multicolumn{3}{c|}{128}                                                                                                      & \multicolumn{1}{c|}{256}                     & 128                                       \\
Obs. embedding                          & \multicolumn{3}{c|}{16}                                                                                                       & \multicolumn{1}{c|}{32}                      & 128                                       \\
Actions embedding                       & \multicolumn{3}{c|}{16}                                                                                                       & \multicolumn{1}{c|}{32}                      & 16                                        \\
Hidden layers after LSTM                & \multicolumn{3}{c|}{{[}128,128{]}}                                                                                            & \multicolumn{1}{c|}{{[}512,256{]}}           & {[}512, 256, 128{]}                       \\ \hline
\end{tabular}
\caption{Hyperparameters table.}
\label{fig:hyper_table}
\end{figure*}

\begin{figure*}[!h]
    \centering
    \includegraphics[width=\linewidth]{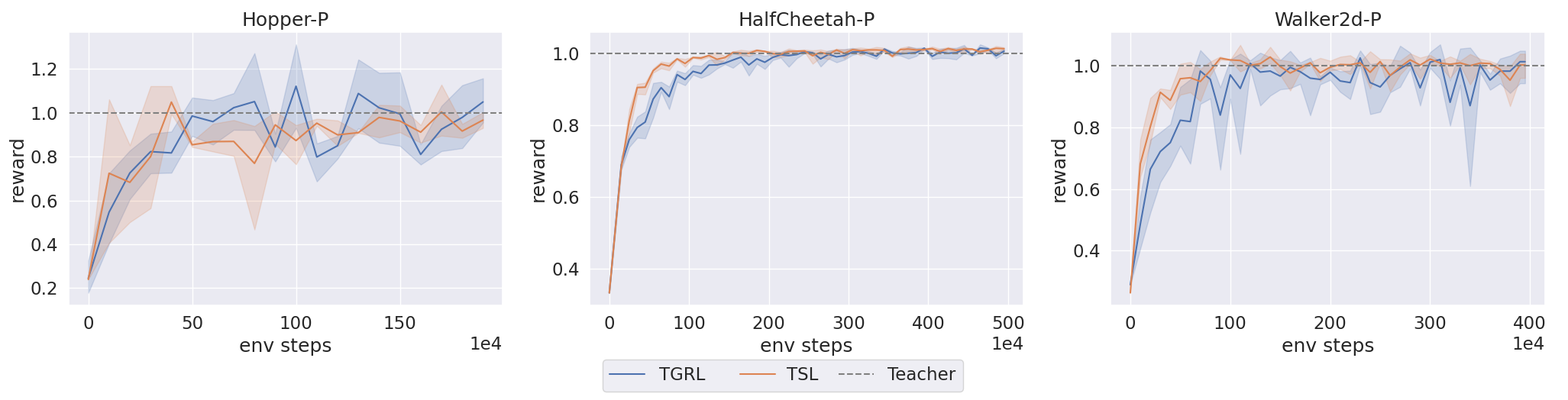}
    \caption{TGRL versus Teacher-Student Learning on domains without information difference. The rewards are normalized based on the teachers' performance.}
    \label{fig:no_imitation_gap}
\end{figure*}

\begin{figure*}[!h]
    \centering
    \includegraphics[width=\linewidth]{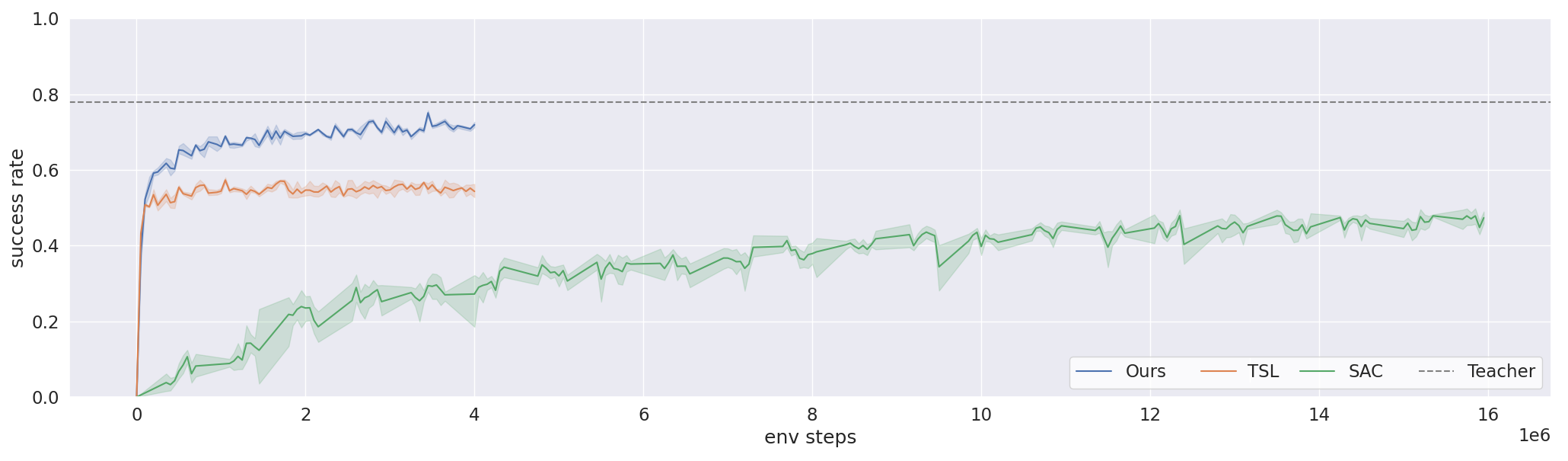}
    \caption{Full training curve of \textit{Shadow Hand} pen reorientation with tactile sensors task}
    \label{fig:full_shadow_hand}
\end{figure*}

\begin{figure*}[!h]
    \centering
    \includegraphics[width=\linewidth]{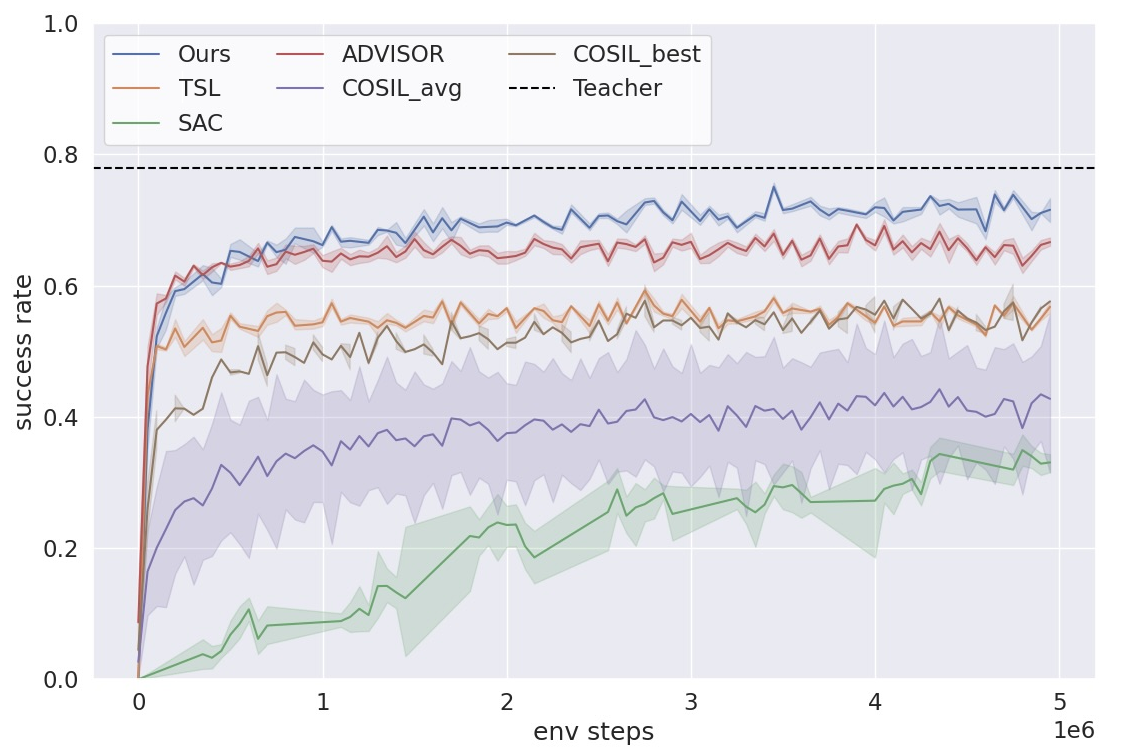}
    \caption{Comparison to baselines of \textit{Shadow Hand} pen reorientation with tactile sensors task}
    \label{fig:full_shadow_hand}
\end{figure*}

\end{document}